\documentclass[letterpaper,twocolumn]{article}
\date{}
\usepackage[hyphens]{url}
\usepackage{graphicx}
\usepackage{natbib}
\usepackage{booktabs}
\usepackage{amsmath}
\usepackage{amssymb}
\usepackage{amsthm}

\newtheorem{definition}{Definition}
\newtheorem{theorem}{Theorem}
\newtheorem{proposition}{Proposition}

\title{Interventions in Closed-Loop Learning: When Does Decision Change Translate into Performance Gain?}
\author{Xuening Wu$^{1}$, Shan Yu$^{2}$, Shenqin Yin$^{3,\dagger}$\\
\small $^{1}$Pfizer, Shanghai, China\\
\small $^{2}$Independent Researcher, Hangzhou, China\\
\small $^{3}$Institute of Humanities and Social Science Data, Fudan University, Shanghai, China\\
\small $^{\dagger}$Corresponding author: ysq@fudan.edu.cn}

\DeclareFontShape{OMS}{cmsy}{m}{n}{<->cmsy10}{}
\input{umsb.fd}\input{umsa.fd}
\DeclareFontShape{U}{msb}{m}{n}{<->msbm10}{}
\DeclareFontShape{U}{msa}{m}{n}{<->msam10}{}
\begin{document}
\maketitle

\begin{abstract}
Feedback guides decisions across reinforcement learning, Bayesian
optimization, and collective routing, and raises related questions for
ReAct agents and recursive self-improvement. When progress stalls, what
additional feedback can change how a system makes decisions, and when
does that change improve performance?
We propose an \emph{information--reorganization--escape} principle that
distinguishes changes in a system's state from changes in how it responds
to the same inputs. Under stable internal feedback, a drift bound shows how the
initial-state term in the expected-deviation bound decays. An event-level
relative-entropy bound gives the minimum divergence consistent with a
specified increase in outcome probability. To test the principle, we compare
intervention with matched continuation and identify the smallest tested
budgets at which decision change and performance gain can each be detected.
In sparse-reward reinforcement learning, a policy change was detectable
at a smaller tested budget than the first certified performance gain.
Misaligned feedback changed the policy without observed held-out success.
In targeted Hartmann--6 optimization, the same ordering held under adjusted
bootstrap intervals; conservative distribution-free bounds did not certify
a decision threshold. In simulated ant routing, transferring the pheromone
map at fixed recipient load improved subsequent delivery. These findings provide partial evidence for the principle and
identify one mechanism through which an early intervention can influence
later performance.
\end{abstract}

\section{Introduction}

A learner that has stopped improving can still respond strongly to feedback.
In a ReAct agent, an observation can redirect the next action within a task
\citep{yao2022react}. Recursive self-improvement poses a longer-horizon question:
can feedback improve the procedure by which a system makes subsequent
improvements? A coding agent that can revise its own editing tools offers
a concrete setting for this question \citep{zhang2025darwin}.

The distinction between a changed decision and a better outcome also arises
in reinforcement learning and Bayesian optimization. A policy may change its
action probabilities, or an optimizer may redirect its next query, while task
performance stays near its plateau. Convergence under repeated updates provides
a starting point for understanding saturation
\citep{robbins1951stochastic,kushner2003stochastic,borkar2008stochastic}.
To ask when additional feedback makes better outcomes accessible, we compare
intervention and continuation from the same checkpoint. Which decisions
change, what does the system subsequently achieve, and what carries any gain
forward after the intervention ends?

We propose an \emph{information--reorganization--escape} principle. Under a
fixed governing structure, internal feedback can move the current knowledge
state but cannot change the transition law that defines its attracting regimes.
External information may reorganize this law, changing the stability,
accessibility, or quality alignment of those regimes. Its strongest dynamical
hypothesis is \emph{attractor replacement}: the old regime loses stability and
a new one attracts. Accessibility may instead change while the old regime
remains stable. Neither follows from large displacement; our experiments test
reorganization and quality conversion, not replacement.

This principle motivates a testable prediction for aligned interventions on a
fixed budget grid. The decision probe may cross its margin at a smaller
tested budget $m_G$ than the first budget $m_Q$ with a certified quality
gain. The ordering depends on the probe, the quality event, and the precision
of their estimates. RL and BO test this budget prediction. In ant routing,
we test a separate causal prediction by exchanging the shared pheromone map
while holding the recipient's recent load fixed. This manipulation identifies
a state carrier in routing; the corresponding carriers in RL and BO remain open.

Our contributions are:
\begin{itemize}
    \item a three-level model that identifies governing structure up to
    observational equivalence on a pre-specified probe;
    \item a formal separation of probe-relative reorganization, accessibility
    change, attractor replacement, and quality-aligned escape;
    \item an operational claim hierarchy and sufficient escape conditions that
    connect classical Lyapunov drift and local stability to an exact
    baseline-relative event information cost;
    \item empirical tests of observable predictions in prospectively frozen
    30-seed RL and BO studies, complemented by a retrospective
    multiplicity-adjusted analysis, continuous quality contrasts, and complete
    decision-margin sensitivity curves;
    \item a state-transfer experiment that isolates how an early routing
    intervention changes later delivery through a shared pheromone map.
\end{itemize}
The component inequalities are classical; the contribution is an identification
and testing framework exposing non-implications and distinct intervention thresholds.

\begin{figure*}[t]
\centering
\includegraphics[width=\textwidth]{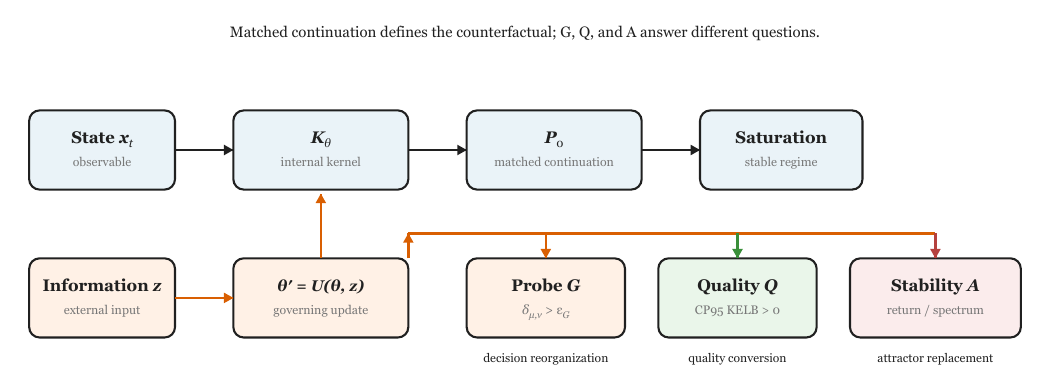}
\caption{Operational decomposition of closed-loop saturation and escape.
Matched continuation under $K_\theta$ defines the counterfactual $P_0$.
External information may update the effective governing parameter, but three
separate tests are required: a probe-relative decision discrepancy $G$, a
quality-event certificate $Q$, and a stability diagnostic $A$. Arrows from the
update branch denote parallel evidential questions, not logical implication.}
\label{fig:framework}
\end{figure*}

\section{Closed-Loop Knowledge Dynamics}

Let $x_t\in\mathcal X$ denote the observable knowledge state: a policy, belief,
draft, or search state. Conditional on a governing parameter
$\theta\in\Theta$ and feedback content $\omega$, evolution follows
\begin{equation}
x_{t+1}\sim K_\theta^\omega(\cdot\mid x_t).
\label{eq:kernel}
\end{equation}
Internal feedback has law $\nu_\theta^{\rm int}(d\omega\mid x)$ and induces
$K_\theta^{\rm int}=\int K_\theta^\omega\,d\nu_\theta^{\rm int}$. Attractors,
invariant sets, and basins are derived properties of this kernel. External
information $z$ may update the governing parameter through
$\theta'=U(\theta,z)$; a state update with $\theta'=\theta$ remains a
state-level perturbation.

Different parameter values can be observationally indistinguishable. Fix
before outcomes a state probe $\mu$, feedback probe $\nu$, metric $d$, and
smallest meaningful effect $\varepsilon_G$. Define
\begin{equation}
\delta_{\mu,\nu}(\theta,\theta')=
\mathbb E_{x\sim\mu,\,\omega\sim\nu(\cdot\mid x)}
W_d\!\left(K_\theta^\omega(\cdot\mid x),
K_{\theta'}^\omega(\cdot\mid x)\right).
\label{eq:discrepancy}
\end{equation}
Write $\theta\sim_{\mu,\nu}\theta'$ when this discrepancy is zero and let
$\mathfrak G_{\mu,\nu}(\theta)$ be the resulting equivalence class.

The probe support is part of the estimand. A full-domain probe asks whether an
intervention redistributes decisions anywhere the system may act; an
incumbent-local probe asks whether it changes behavior near the currently
visited regime. Neither dominates the other. Reporting a discrepancy without
the support, feedback coupling, ground metric, and meaningful effect margin is
therefore incomplete. Figure~\ref{fig:framework} makes these identification
boundaries explicit.

Operationally, probe design should follow the downstream decision whose change
would make the intervention scientifically consequential. The RL probe therefore
balances task phases rather than replaying only states visited by the saturated
policy, while the BO probe covers the full candidate domain rather than only the
incumbent neighborhood. These choices make global accessibility visible, but the
resulting claims remain tied to those supports; the incumbent-local BO audit
explicitly demonstrates this dependence.

\begin{definition}[Probe-Relative Reorganization]
An intervention produces governing-structure reorganization at resolution
$\varepsilon_G$ when $\theta'=U(\theta,z)$ and a confidence lower bound for
$\delta_{\mu,\nu}(\theta,\theta')$ exceeds $\varepsilon_G$. When only a
decision-level projection is observed, the conclusion is restricted to that
projection and is a candidate reorganization of the full structure.
\end{definition}

Symmetrically, an upper confidence bound below $\varepsilon_G$ certifies
probe-relative equivalence at the declared resolution. An interval crossing
$\varepsilon_G$ is inconclusive, rather than evidence for either class.

This definition prevents a large state jump from being labeled structural.
It also does not imply utility. Two derived mechanisms must be distinguished:
\emph{attractor replacement}, where the old attracting set becomes unstable
under $K_{\theta'}^{\rm int}$ and a new set becomes attracting; and
\emph{accessibility reorganization}, where transition probabilities or basin
geometry change while the old set may remain stable. The former requires a
stability test such as perturbation--return analysis or a controlled local
spectral estimate in addition to Equation~\ref{eq:discrepancy}.

\subsection{Claim Hierarchy and Non-Implications}

Let $D$ denote a large observed state displacement, $G$ a positive
probe-relative reorganization test, $P$ persistence under matched
post-intervention continuation, $Q$ a positive quality-event KELB certificate,
and $A$ an attractor-replacement test. Without additional assumptions,
\begin{equation}
D\not\Rightarrow G,\qquad
G\not\Rightarrow Q,\qquad
Q\not\Rightarrow G,\qquad
(G\wedge P)\not\Rightarrow A.
\label{eq:nonimplications}
\end{equation}
A state kick gives the first counterexample; a matched but misaligned kernel
change gives the second; a quality-improving state intervention gives the third;
and crossing a pre-existing basin boundary gives the fourth. In particular,
persistence after crossing does not show that the regime left behind became
unstable.
Accordingly, an operational report freezes the state metric, probed kernel
projection, structural margin, quality event and baseline, and any stability
diagnostic. This prevents movement or reward improvement from being mistaken
for a stronger mechanism.

\begin{table*}[t]
\centering
\setlength{\tabcolsep}{5pt}
\begin{tabular}{p{.15\textwidth}p{.21\textwidth}p{.24\textwidth}p{.31\textwidth}}
\toprule
Claim & Frozen estimand & Positive criterion & What it does not establish \\
\midrule
State displacement $D$ & State/output metric & Effect exceeds a state margin & A changed decision law, persistence, or utility \\
Decision reorganization $G$ & Kernel projection and probe support & Discrepancy LCB exceeds $\varepsilon_G$ & Full-kernel change, quality gain, or lost stability \\
Quality conversion $Q$ & Event $E_Q$ and matched $P_0$ & Adjusted KELB lower bound is positive & Reorganization or attractor replacement \\
Persistence $P$ & Matched post-intervention continuation & Effect survives a frozen horizon & Whether the old attractor remains locally stable \\
Attractor replacement $A$ & Return map or local operator & Old regime becomes unstable and a new one attracts & Global uniqueness or semantic superiority \\
\bottomrule
\end{tabular}
\caption{Claim hierarchy and minimum operational evidence. Each row requires a
separately frozen estimand; no claim in one row follows from another without
additional assumptions.}
\label{tab:claim-hierarchy}
\end{table*}

\section{Saturation Under Fixed Structure}

For fixed $\theta$, let $V:\mathcal X\to\mathbb R_{\ge0}$ measure distance or
deficiency relative to a stable regime.

\begin{theorem}[Saturation Under Lyapunov Drift]
Suppose there are $\rho_t\in[0,1)$ and $\sigma_t\ge0$ such that
\[
\mathbb E[V(x_{t+1})\mid x_t]\le \rho_tV(x_t)+\sigma_t.
\]
Then
\[
\mathbb E V(x_t)\le
\left(\prod_{s<t}\rho_s\right)V(x_0)+
\sum_{s<t}\left(\prod_{r=s+1}^{t-1}\rho_r\right)\sigma_s.
\]
If $\rho_t=\rho$ and $\sigma_t=\sigma$, then
$\mathbb EV(x_t)\le\rho^tV(x_0)+\sigma/(1-\rho)$.
\end{theorem}

\begin{proof}
Take expectations in the one-step inequality and unroll the recursion. The
constant case follows from the geometric-series bound.
\end{proof}

Thus saturation is not necessarily convergence to a point: it is attenuation
of the transient toward a noise-controlled stability region. In differentiable
coordinates, a local Jacobian $J$ with spectral radius below one supplies a
contraction diagnostic; eigenvectors identify slow observable directions when
their eigenvalues are close to one. This criterion is local and
coordinate-dependent, not a universal semantic convergence test.

\subsection{Local Spectral Diagnostics}

Let $F_\theta(x)=\mathbb E[x_{t+1}\mid x_t=x]$ and
$F_\theta(x_\theta^*)=x_\theta^*$. If
$J_\theta=DF_\theta(x_\theta^*)$ has spectral radius below one, the mean map
contracts in an equivalent local norm; near-unit eigenvalues identify slow
observable modes \citep{khalil2002nonlinear,strogatz2015nonlinear}. Estimated
spectra nevertheless depend on coordinates, support, horizon, and estimation error.

For attractor replacement, a useful sufficient pattern is stronger: the old
fixed point is stable before intervention, an identified eigenvalue crosses the
unit circle afterward, and a distinct fixed point has a post-intervention
spectrum inside the unit circle. Persistence at a new behavior without the
first two tests is compatible with accessibility reorganization instead, as
distinguished in Table~\ref{tab:claim-hierarchy}.

\section{Reorganization and Quality-Aligned Escape}

Write $a=x_\theta^*$ for the old center and $b=x_{\theta'}^*$ for the new
center after a verified structural update. Let $c=\sigma/(1-\rho)$ be the
old noise radius and $\Delta=d(a,b)$.

\begin{proposition}[Metric Escape Condition]
Let $x$ satisfy $d(x,a)\le c$. Suppose the updated dynamics satisfy the
one-step bound at this initial state,
\[
\mathbb E[d(x_{t+1},b)\mid x_t=x]\le \rho' d(x,b)+\sigma_{\rm ext},
\]
with $0\le\rho'<1$ and $\sigma_{\rm ext}\ge0$. If
\[
\Delta>\frac{c(1+\rho')+\sigma_{\rm ext}}{1-\rho'},
\]
then $\mathbb E[d(x_{t+1},a)\mid x_t=x]>c$.
\end{proposition}

\begin{proof}
The reverse triangle inequality gives $d(x_{t+1},a)\ge
\Delta-d(x_{t+1},b)$. Taking conditional expectations and using
$d(x,b)\le c+\Delta$ gives the lower bound
$(1-\rho')\Delta-\rho'c-\sigma_{\rm ext}>c$.
\end{proof}
The conclusion holds throughout the old ball if the one-step bound
holds on that entire ball. It concerns expected distance after one update;
almost-sure escape and attractor replacement require separate evidence.

Geometric movement is not a quality guarantee. Let $P_0$ be matched
no-intervention continuation, $P_z$ the intervention law, and $E_Q$ a
pre-specified quality event. Write $p_0=P_0(E_Q)$ and $p_z=P_z(E_Q)$.

\begin{definition}[Knowledge Escape Lower Bound]
For $\eta\in[0,1-p_0]$, define
\[
\operatorname{KELB}_{E_Q}(\eta)=
\inf_{\substack{P\ll P_0\\P(E_Q)\ge p_0+\eta}}
\operatorname{KL}(P\Vert P_0).
\]
\end{definition}

\begin{theorem}[Exact Event-Level KELB]
\label{thm:kelb}
\[
\operatorname{KELB}_{E_Q}(\eta)
=\operatorname{kl}(p_0+\eta\Vert p_0)\ge2\eta^2,
\]
where $\operatorname{kl}$ is binary relative entropy.
\end{theorem}

\begin{proof}
Set $q=p_0+\eta$. For $0<p_0<1$, data processing on the event
indicator and monotonicity of $\operatorname{kl}(u\Vert p_0)$ for
$u\ge p_0$ give the lower bound. Equality follows by multiplying
the baseline density by $q/p_0$ on $E_Q$ and by
$(1-q)/(1-p_0)$ on its complement. For $\eta=0$, $P=P_0$
attains zero, including at $p_0=0,1$. If $p_0=0$ and $\eta>0$,
absolute continuity makes the feasible set empty and both sides of the
equality are $+\infty$. Pinsker's inequality gives the stated relaxation.
\end{proof}

Zero observed baseline successes in a finite sample leave uncertainty about the population value of $p_0$.
The exact form retains baseline difficulty, unlike the Pinsker relaxation.
Fix an intervention $z$ and a common initial-state law. Let
$h_t=(x_0,\ldots,x_t)$, let $\mathbb P_z^{0:T}$ denote the joint trajectory
law induced by conditional kernels $K_t^z(\cdot\mid h_t)$, and let
$\mathbb P_0^{0:T}$ and $K_t^0$ denote the matched no-intervention counterparts.
For a horizon-$T$ event $E_Q^{\le T}$, set
$p_{z,T}=\mathbb P_z^{0:T}(E_Q^{\le T})$ and define $p_{0,T}$ analogously.
Assuming $\mathbb P_z^{0:T}\ll\mathbb P_0^{0:T}$, data processing and the KL
chain rule give
\begin{align}
\operatorname{KL}(\mathbb P_z^{0:T}\Vert\mathbb P_0^{0:T})
&\ge \operatorname{kl}(p_{z,T}\Vert p_{0,T}),\nonumber\\
\operatorname{KL}(\mathbb P_z^{0:T}\Vert\mathbb P_0^{0:T})
&=\mathbb E_{\mathbb P_z^{0:T}}\sum_{t<T}\operatorname{KL}
(K_t^z(\cdot\mid h_t)\Vert K_t^0(\cdot\mid h_t)).
\label{eq:trajectory-kelb}
\end{align}
Thus a finite-horizon event certificate lower-bounds cumulative
baseline-relative information, although equality need not be attainable by a
restricted intervention family.

KELB also differs from mutual information. If $Z$ randomizes interventions,
$P_z$ is the next-state law, and $\bar P=\mathbb E_ZP_Z$, then
\begin{equation}
\mathbb E_Z\operatorname{KL}(P_Z\Vert P_0)
=I(x_{t+1};Z\mid x_t)+\operatorname{KL}(\bar P\Vert P_0).
\label{eq:mi-decomposition}
\end{equation}
This follows by adding and subtracting $\log d\bar P$. The first term measures
variation among signal-conditioned updates; the second measures average shift
from matched continuation. A deterministic useful intervention can therefore
have zero mutual information but a nonzero baseline shift.

For finite samples, write the Clopper--Pearson marginal intervals as
$C_z=[L_z,U_z]$ and $C_0=[L_0,U_0]$, with their failure probabilities
allocated across the declared comparison family. The improvement certificate is
\begin{equation}
\underline{\operatorname{KELB}}=
\begin{cases}
\operatorname{kl}(L_z\Vert U_0),& L_z>U_0,\\
0,& L_z\le U_0.
\end{cases}
\label{eq:finite-kelb}
\end{equation}
Thus the lower certificate is positive only when the entire rectangle supports
$p_z>p_0$. Zero is also returned when the rectangle lies wholly in the
non-improvement region; no minimization over an empty improvement set is
required. An empirical zero baseline is not treated as a known zero
probability. This is a bound on distributional divergence, not a measurement
of the information supplied by the intervention.

\begin{definition}[Two Intervention Thresholds]
For a pre-specified intervention family $\{z_m:m\in\mathcal M\}$ on a frozen
discrete budget grid $\mathcal M$, define
\begin{align*}
m_G&=\min\{m\in\mathcal M:\operatorname{LCB}_{1-\alpha}
(\widehat\delta_{\mu,\nu,m})>\varepsilon_G\},\\
m_Q&=\min\{m\in\mathcal M:\underline{\operatorname{KELB}}_{1-\alpha}
(E_Q;z_m)>0\}.
\end{align*}
The cases $m_G=m_Q$, $m_G<m_Q$, and finite $m_G$ with no certified $m_Q$
denote coincident certification, separated certification, and probe-relative
reorganization without certified conversion, respectively. Missing certificates
are reported as absent within the tested grid, not assigned a population
threshold of infinity.
\end{definition}

These are the smallest \emph{certified tested budgets}, not estimates of a
continuous phase-transition point. Certification order depends on probe
resolution, sample size, and interval construction; it is not by itself a
temporal or causal ordering of mechanisms. Moreover, $m$ counts domain-specific
intervention units---expert-labeled state--action samples in RL and injected observations in BO---so
its numerical value is not comparable across domains and is not a direct
measure of information bits.

No ordering is assumed without additional monotonicity and alignment
conditions. In particular, a positive KELB is neither necessary nor sufficient
for reorganization; it certifies the beneficial consequence of a candidate
reorganization.

The decomposition yields three falsifiable predictions. First, weak aligned
information may cross $m_G$ before it changes a quality event, producing
$m_G<m_Q$. Second, misaligned information can create a still larger decision
discrepancy while leaving KELB at zero. Third, $m_G$ can change when the probe
support changes even though the underlying trajectories do not; this is an
estimand change, not failed robustness. By contrast, candidate-set size,
projection count, and statistically equivalent discrepancy estimators should
not change the classification once they adequately approximate the same frozen
probe.

\begin{table*}[t]
\centering
\small
\setlength{\tabcolsep}{4pt}
\begin{tabular}{lccccc}
\toprule
Condition & Policy TV [adjusted CI] & Mean success & Quality events & Adjusted KELB & LOO min. \\
\midrule
PPO-only & --- & .000 & 0/30 & --- & --- \\
Targeted, $m=1000$ & .601 [.563,.639] & .000 & 0/30 & .000 & .000 \\
Targeted, $m=10000$ & .718 [.683,.746] & .225 & 30/30 & .412 & .380 \\
Targeted, $m=20000$ & .796 [.763,.821] & .660 & 30/30 & .412 & .380 \\
Mismatched, $m=20000$ & .825 [.808,.846] & .000 & 0/30 & .000 & .000 \\
\bottomrule
\end{tabular}
\caption{Confirmatory RL operational validation after matched sparse-reward
continuation. TV is measured on the frozen balanced task-phase support; quality
events use the pre-specified success threshold $.10$. Decision intervals are
approximate, adjusted across 64 comparisons at family error $.025$; quality
bounds cover 124 comparisons at family error $.025$. LOO denotes the minimum
adjusted lower bound across leave-one-seed-out refits, a stability diagnostic.}
\label{tab:rl}
\end{table*}

\begin{figure*}[t]
\centering
\includegraphics[width=\textwidth]{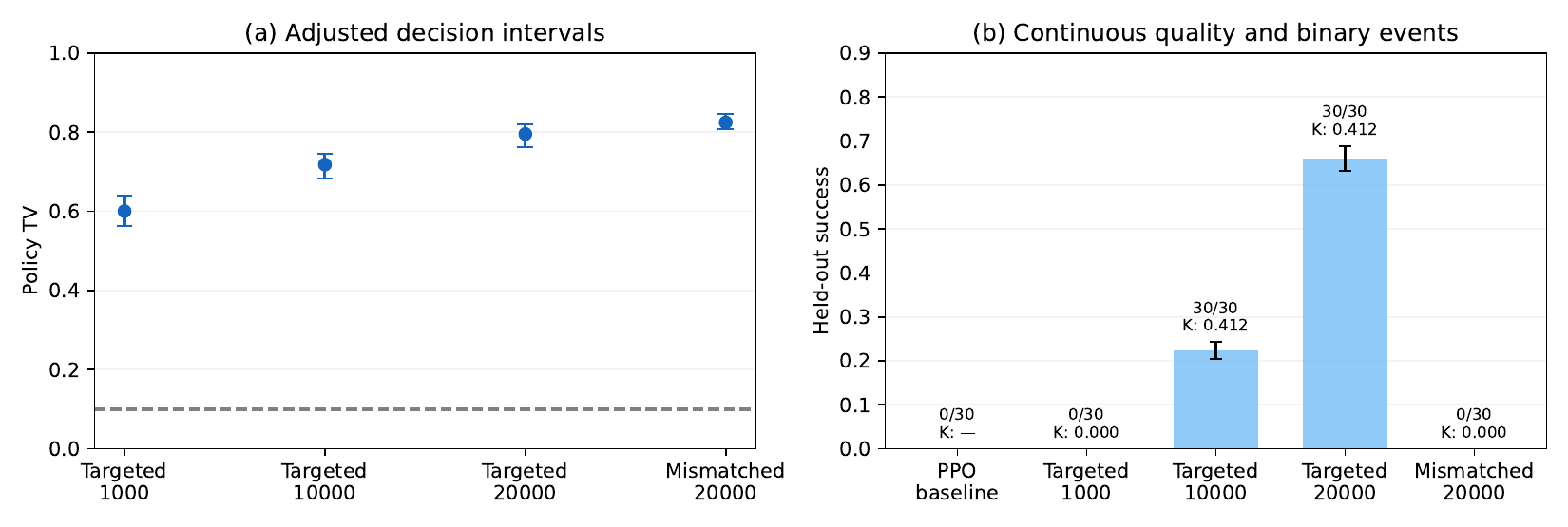}
\caption{RL results under the revised multiplicity analysis. (a) Mean policy TV
and adjusted approximate seed-bootstrap intervals on the frozen balanced
support; all interventions exceed $\varepsilon_G=.10$. (b) Mean held-out
success and adjusted approximate intervals for the continuous-outcome family;
labels give binary quality counts and adjusted KELB. Zero-width success
intervals for all-zero samples reflect bootstrap degeneracy, not population
certainty. The targeted budgets $10000$ and $20000$ share the same binary
event count but differ substantially in continuous success.}
\label{fig:rl-thresholds}
\end{figure*}

\section{Prospective Operational Validation}

Both studies generated outcomes on 30 confirmatory training seeds under
frozen protocols, excluding calibration data from the reported comparisons.
The original estimands, branch construction, and event definitions are retained.
The multiplicity adjustment, continuous-quality summaries, and complete margin
curves reported here are retrospective analyses of those archived outcomes;
the stricter BO same-query event is explicitly exploratory.

\subsection{Frozen Estimands and Revised Certification}
The training seed is the independent statistical unit, rather than an episode,
query, or probed state. Branches sharing a seed remain paired in resampling.
The quality family contains 124 comparisons: 60 BO task--condition--budget
comparisons for each of two event definitions, and four RL comparisons at the
frozen success cutoff $.10$. We allocate family error $.025$, giving each
comparison error $.025/124$ and each of its two two-sided CP intervals error
$.025/(2\cdot124)$. Repeated baselines are counted conservatively. For fixed
events and iid seed outcomes, the family has coverage at least $.975$;
including the exploratory event in this declared family does not eliminate
unrecorded event-selection effects.

The decision family contains 60 BO and four RL primary-probe means, with
family error $.025$. One million common seed-bootstrap draws within each
domain provide percentile endpoints at $.025/(2\cdot64)$ and its complement.
These intervals have approximate coverage; combining them with CP bounds does
not establish exact joint $.95$ coverage. A distribution-free sensitivity uses
Hoeffding intervals for the bounded probe values. These retain the RL decision
certificates, but certify no BO decision threshold at nonnegative margins.

RL quality thresholds additionally require every adjusted leave-one-seed-out
refit to remain positive, preserving the frozen RL rule. BO reports the first
positive quality certificate and separately records leave-one-out stability.
Leave-one-out checks are stability diagnostics, not an additional simultaneous
coverage claim. The original BO pointwise Styblinski--Tang generic threshold
was $2$, whereas its leave-one-out-stable threshold was $4$; the previous table
reported the latter. The revised family-adjusted threshold is $4$.

For continuous RL success, a separate family covers five means and all ten
paired contrasts at error $.05$, using one million common seed resamples.
These intervals are approximate and are not included in the preceding
quality/decision error allocation.

The margin $\varepsilon_G$ is a declared probe resolution, not a confidence
level or a theoretically identified natural boundary. RL froze $.10$ after
considering $.02,.05,.10$ in calibration; BO froze $.01$ with $.005,.02$
sensitivity values. We report the complete piecewise curve
$m_G(\varepsilon_G)$ over $[0,1]$, holding the revised bounds and $m_Q$ fixed.
For a finite quality threshold, separated certification holds exactly when
\[
0\le\varepsilon_G<\max_{m<m_Q} L_m,
\]
provided this maximum is positive; an empty set yields no such range.
This deterministic transformation of simultaneous bounds requires no new
per-margin tests, but inherits the bounds' coverage limitations. Broad ranges
support robustness to margin choice, not a causal dynamical explanation.
Original pilot and pointwise robustness results are retained as historical
analyses in the supplement, separately from the revised evidence.

\subsection{Sparse-Reward Reinforcement Learning}

We study the $8\times8$ DoorKey task from MiniGrid. PPO is evaluated every
$10{,}000$ environment steps and declared saturated when the success-rate
range over five consecutive evaluations is below $.05$. Each saturated checkpoint branches
into matched PPO-only continuation, targeted expert-action behavior cloning
(BC) at $m\in\{1000,10000,20000\}$, and mismatched BC at $m=20000$ using the
fixed action permutation $[1,2,0,4,5,3,6]$ globally across states, seeds, and
time. BC observations, epochs, batches,
and learning rates are matched. Every branch then receives $100{,}000$ PPO
steps under the original sparse reward with matched randomness and 100 common
held-out layouts.

At a saved checkpoint, policy parameters induce the decision-kernel projection
\[
\begin{aligned}
K_\theta^{\rm dec}(a\mid s)&=\pi_\theta(a\mid s),\\
\delta_\mu^{\rm dec}(\theta,\theta')
&=\mathbb E_{s\sim\mu}\bigl[
\operatorname{TV}(\pi_\theta(\cdot\mid s),
\pi_{\theta'}(\cdot\mid s))\bigr].
\end{aligned}
\]
The frozen support contains 256 observations from each of the pre-key,
key-to-door, and door-to-goal phases, and $\varepsilon_G=.10$. This estimates a policy-decision projection of
Equation~\ref{eq:discrepancy}, not the full environment--learner transition
kernel. Accordingly, $m_G$ is a decision-reorganization threshold and only
candidate evidence about the full governing structure. The seed-level quality
event is held-out success at least $.10$.

All 30 baselines have zero observed held-out success. Targeted $m=1000$
produces a large policy discrepancy after matched continuation, but no quality
event. Targeted $m=10000$ and $m=20000$ produce 30/30 events against 0/30
baselines, with family-adjusted KELB $.411636$ and minimum leave-one-out bound
$.379521$. Thus $m_G=1000<m_Q=10000$ remains supported on the tested targeted
grid (Figure~\ref{fig:rl-thresholds}). The ordering persists for
$0\le\varepsilon_G<.562859$ under the adjusted bootstrap bounds and for
$0\le\varepsilon_G<.223250$ under the conservative Hoeffding bounds.

The binary event does not exhaust the quality response. Mean success is
$.224667$ at $m=10000$ and $.660333$ at $m=20000$. Their paired difference is
$.435667$, with adjusted approximate interval $[.403333,.466667]$.
At $m=20000$, targeted minus mismatched success is $.660333$
$[.632000,.688333]$. All-zero samples yield degenerate percentile intervals;
these do not establish that population success is exactly zero. Likewise, the
original binary event contrast of $1$ with bootstrap interval $[1,1]$ is a
boundary-sample result, not uncertainty-free inference.

Mismatched feedback has a larger observed mean TV than targeted feedback at
$m=20000$ ($.825$ versus $.796$), but no observed held-out success. A separate
paired test of that TV difference is not claimed here. The contrast supports
the prediction that substantial decision displacement need not produce a
quality benefit. Together with the separated targeted thresholds, this is
partial empirical support for the principle's observable predictions.
Because BC is a pulse followed by the original PPO update rule, these data do
not distinguish old-attractor destabilization from movement into another regime.

\subsection{Bayesian Optimization}
BO begins in a restricted local trust region and then reopens the full domain
for every branch. Each branch receives 12 real expected-improvement queries.
We test $m\in\{1,2,4,8,16\}$ injected observations across baseline, generic,
targeted, noisy-targeted, and mismatched feedback. Injected observations update
the Gaussian-process posterior but are neither counted as real queries or
improvements nor removed by query de-duplication; BO must actively verify them.
For seed $s$, condition $c$, and budget $m$, define
\begin{align}
A_{s,c,m}&=\mathbb I[\exists t:
x^{\rm query}_{s,c,m,t}\in\mathcal B_{\rm global}],\nonumber\\
E^Q_{s,c,m}&=\mathbb I[A_{s,c,m}=1\ \text{and}\
R_{s,c,m}\le R_{\rm local}],
\label{eq:bo-event}
\end{align}
where $R_{s,c,m}$ is best post-intervention query regret and $R_{\rm local}$ is
regret at the nominal Phase-1 local attractor. This is a trajectory-level
conjunction: the basin hit and qualifying regret need not occur at the same
query. The exploratory stricter event requires both properties at one real
query. The decision probe compares the
normalized expected-improvement weights before and after feedback using
weighted sliced Wasserstein distance on a frozen full-domain support and 32
fixed coordinate/random projections, with $\varepsilon_G=.01$.
Gaussian-process uncertainty-guided search
connects this setting to classical BO \citep{srinivas2010gaussian}.

\begin{figure*}[t]
\centering
\includegraphics[width=\textwidth]{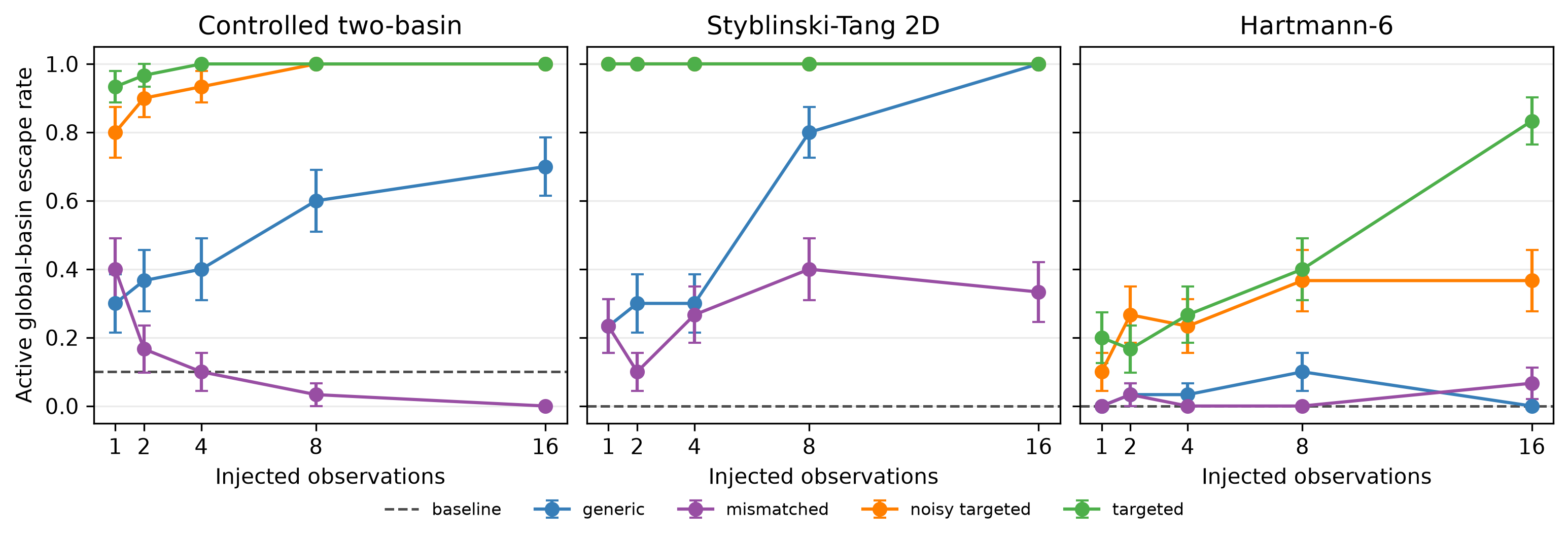}
\caption{Task-specific feedback-strength responses in BO. Markers and whiskers
show mean $\pm$ SEM active global-basin escape rates over 30 paired
confirmatory seeds; dashed lines are matched no-feedback baselines after the
domain is reopened. This diagnostic visualizes alignment and task dependence;
the stricter quality event used for $m_Q$ additionally requires regret no worse
than the nominal local optimum and is certified by KELB.}
\label{fig:bo-escape}
\end{figure*}

\begin{table}[t]
\small
\centering
\setlength{\tabcolsep}{1.5pt}
\begin{tabular}{llcc}
\toprule
Task & Feedback & $m_G$ & $m_Q$ \\
\midrule
Two-basin & Generic / Mismatched & 8 / --- & --- / --- \\
 & Noisy-targeted / Targeted & 1 / 1 & 16 / 2 \\
Hartmann--6 & Generic / Mismatched & 1 / 1 & --- / --- \\
 & Noisy-targeted / Targeted & 1 / 1 & --- / 16 \\
Styblinski--Tang & Generic / Mismatched & 1 / 1 & 4 / --- \\
 & Noisy-targeted / Targeted & 1 / 1 & 1 / 1 \\
\bottomrule
\end{tabular}
\caption{Revised BO thresholds for the original frozen trajectory event.
$m_G$ uses family-adjusted approximate bootstrap intervals; $m_Q$ uses
family-adjusted CP bounds. Dashes indicate no certificate for $m\le16$, not
absence of an effect. Conservative Hoeffding bounds certify no BO $m_G$.}
\label{tab:bo}
\end{table}

Table~\ref{tab:bo} retains coincident and separated certification, as well as
decision changes without a quality certificate on the tested grid. The complete
margin analysis makes the dependence on probe resolution explicit. With the
original quality event, $m_G<m_Q$ holds for Two-basin targeted when
$0\le\varepsilon_G<.207724$, Two-basin noisy-targeted below $.189651$,
Hartmann--6 targeted below $.030446$, and Styblinski--Tang generic below
$.156398$. No ordering is asserted for a missing quality certificate.

For Hartmann--6 targeted, the numerical value $m_G=1$ is sensitive near the
primary margin: it holds only below $.010262$. Nevertheless, as the margin
increases, $m_G$ takes values $2,4,8$ before reaching $m_Q=16$ at $.030446$.
Thus all three frozen sensitivity margins $.005,.01,.02$ preserve the ordering.
This supports a separation over a continuous margin range, rather than only
at the selected primary value, while retaining approximate decision coverage.

Reconstruction from all 1890 branches and 22680 real-query rows finds 69
original-event positives that fail the stricter same-query event, all on
Two-basin; no strict-event positive fails the original event. Six discordances
are baselines, reducing that task's baseline event count from 6/30 to 0/30.
Consequently, event-set inclusion does not order baseline-relative certificates.
The stricter event changes Two-basin targeted $m_Q$ to $1$, eliminating finite
$m_G<m_Q$ there, but leaves the Hartmann--6 targeted range unchanged. These
post-hoc results delimit event interpretation rather than replace the frozen
estimand. Historical probe-support and estimator audits remain in the
supplement; they have not been rerun under every revised comparison family.

\begin{figure*}[t]
\centering
\includegraphics[width=.85\textwidth]{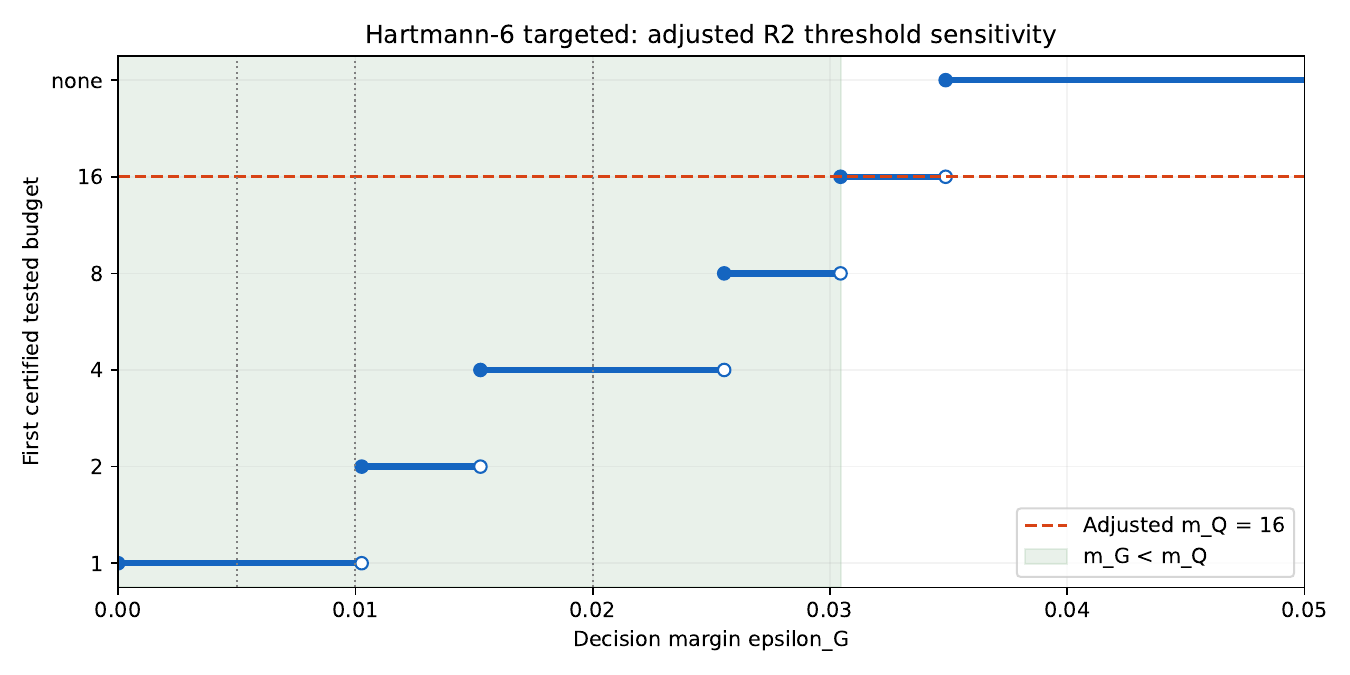}
\caption{Complete adjusted decision-margin sensitivity for Hartmann--6
targeted feedback, zoomed to $[0,.05]$. Green shading marks finite
$m_G<m_Q$ for the original event. Closed/open markers indicate included/excluded
endpoints under $L_m>\varepsilon_G$. The quality threshold remains $16$;
above $.034875$ no decision budget is certified. Breakpoints are numerical
functions of approximate bootstrap bounds, not natural phase boundaries.}
\label{fig:margin-sensitivity}
\end{figure*}

\section{Mechanism Experiment: Transferring Collective Routing Memory}
\label{sec:ants}
The RL and BO comparisons locate decision and quality changes on a budget
grid. A different question is whether an intervention leaves an internal
state that can carry its effect into later decisions. In simulated ant routing
we can exchange that state between paired branches. The routing rule stays
fixed throughout; the pheromone map changes the route probabilities under
that rule.

Each round, 32 ants choose among terminal routes. A route with pheromone
$h$ and previous load $n$ receives weight $(1+h)^{1.5}/(1+n/32)$.
The probability of entering either branch is proportional to its total
weight; conditional route probabilities are proportional to the individual
weights. Only delivered ants deposit pheromone, with $h'=0.85h+d$ for $d$
deliveries on that route. After 32 warm-up rounds we switched capacities and
guided four ants for one round, either to the newly high-capacity route or
to uniformly random routes. We then withdrew guidance.

For each seed we crossed the two resulting pheromone maps with the two
recipient load vectors. Holding load and subsequent random draws fixed, we
estimated the effect of the targeted map in each recipient. Quality was the
fraction delivered per round, averaged over rounds 2--16 after withdrawal.
The original four-route pilot yielded gains of 2.87 and 2.81 percentage
points in the random and targeted recipients, respectively. Its budget sweep
identified neither threshold at the specified five-percentage-point margins.

We fixed three new networks and disjoint sets of 128 seeds before the
follow-up run. A had two routes on one branch and three on the other;
B had three on each. C added a capacity-12 bottleneck shared by one route
from each side of B. Table~\ref{tab:transfer} reports all prespecified
contrasts. Even under the shared bottleneck, retaining the targeted map
increased later delivery with either recipient load. These paired
manipulations identify a causal contribution of collective memory under
the simulated update rule. The design used one intervention budget and a
capacity switch; it does not establish a pre-intervention plateau or locate
$m_G$ and $m_Q$. The state variables carrying the RL and BO effects remain
unidentified. Supplementary Section~S6 reports separate self-play pilots
in which the plateau criterion was never met; Section~S7 specifies the routing networks.

\begin{table}[t]
\centering\small
\begin{tabular}{lcc}\toprule
Network & Random load & Targeted load\\\midrule
A & 3.01 [2.19, 3.83] & 2.89 [2.11, 3.67]\\
B & 1.80 [1.11, 2.49] & 1.71 [1.03, 2.40]\\
C & 1.92 [1.32, 2.53] & 1.89 [1.27, 2.50]\\\bottomrule
\end{tabular}
\caption{Delivery gain in percentage points from substituting the
targeted pheromone map for the random-guidance map at fixed load.
Brackets give paired $t$ intervals with approximate 95\% simultaneous
coverage under Bonferroni adjustment over six comparisons. Columns are
complementary contrasts within the same seeds.}
\label{tab:transfer}
\end{table}

\section{Cross-Domain Evidence Audit}

\subsection{Matched Intervention Contrasts}

In RL and BO, interventions are compared across branches from the same
checkpoint, with the seed as the independent unit of analysis. A valid
baseline preserves every opportunity available after branching while
withholding only the focal information. In RL, all branches resume PPO for the
same number of environment steps under the original sparse reward; BC compute,
evaluation layouts, and continuation randomness are matched across aligned and
misaligned branches. In BO, every branch reopens the same domain and receives
the same real-query budget; injected observations affect the posterior but are
not credited as realized improvements. These controls rule out continued
optimization or reopening alone as explanations for escape.

Alignment is varied at a matched intervention budget. The RL action permutation
preserves the number and format of demonstrations while destroying task
semantics; the BO mismatched condition places observations away from the
pre-specified global basin. Targeted--mismatched contrasts therefore estimate
the consequence of alignment rather than exposure, update count, or branch
duration, but only within the frozen intervention family.

\subsection{What Each Domain Establishes}

The RL experiment establishes $G$ at the decision projection for every tested
intervention, horizon-level $P$ at the frozen $100{,}000$-step endpoint, and
$Q$ only for targeted $m\ge10000$. Targeted feedback thus produces separated certification budgets; mismatched
feedback produces detectable decision displacement without a certified quality gain. It does not establish $A$: success may
cross a pre-existing basin boundary while the failure regime remains locally
stable. Supplementary Section~S1 reports an exploratory perturbation--return
pilot comprising $5$ seeds, $5$ checkpoint conditions, $3$ noise draws, and
$3$ scales ($225$ perturbation--continuation comparisons). Fourteen of 15
condition--scale cells expand in every draw. The boundary cell, targeted
$m=20000$ at scale $.02$, expands in 12/15 draws but remains positive after
seed-first aggregation (geometric ratio $1.217$, 95\% bootstrap CI
$[1.114,1.354]$); quality category returns in 223/225 comparisons. This supports
an aggregate finite-horizon diagnostic, not attractor replacement: confirmation
around both regimes or a controlled local operator estimate remains necessary.

In BO, the adjusted bootstrap intervals certify changes in the acquisition
distribution on the full-domain probe for several intervention conditions.
The primary quality event requires a basin hit and qualifying regret over
the real-query trajectory; different queries can satisfy the two criteria.
Across conditions, quality gains are certified at the first tested budget,
at a larger budget, or at none of the tested budgets. The 12-query horizon shows downstream behavioral effects
but tests neither asymptotic persistence nor attractor stability; the claim is
therefore restricted to the observed acquisition projection.

\subsection{Transporting the Protocol}

Across domains, roles rather than metrics transport: a matched continuation
law, decision-relevant probe, meaningful margin, independently declared quality
event, and alignment control. Policy TV and sliced Wasserstein distance are not
numerically comparable; their common role is to test whether an intervention
changes an observable decision law beyond a frozen resolution. Domain-specific
quality events yield Bernoulli outcomes to which the same event-level KELB
applies.

This separation also specifies failure modes. Local probes may miss global
accessibility changes, whereas overly broad probes may detect effects that
never influence reachable behavior. Data-dependent event selection can invalidate nominal confidence guarantees
unless accounted for; it does not invalidate the underlying event-level
information inequality. Treating episodes or query candidates as independent
replicates understates uncertainty. Large persistent $G$ without a quality
certificate is consistent with the predicted separation, but lack of a
certificate alone does not establish zero quality effect.

\section{Related Work and Positioning}

Lyapunov drift, stochastic approximation, and local spectra characterize convergence
and stability under fixed dynamics
\citep{meyn2009markov,kushner2003stochastic,borkar2008stochastic,
benaim1999dynamics,khalil2002nonlinear,strogatz2015nonlinear}; RL supplies
domain-specific contraction views
\citep{bertsekas1996neuro,sutton2018reinforcement}. Unlike global behavioral or
bisimulation metrics \citep{ferns2011bisimulation}, our discrepancy compares
intervention-indexed laws only on declared probe support. It therefore defines
a deliberately weaker, experiment-specific equivalence class rather than a
global behavioral metric. Likewise, local spectral criteria diagnose stability
near an identified fixed point but do not by themselves identify globally
meaningful regimes or their semantic quality.

Kernel tests and optimal transport detect distributional difference
\citep{gretton2012kernel,villani2008optimal}; we add matched continuation, a
pre-specified meaningful margin, and an explicit probe restriction. These
elements convert generic distributional difference into a test of an identified
intervention contrast. BO studies posterior
exploration and localized search
\citep{srinivas2010gaussian,frazier2018tutorial,eriksson2019scalable}; matched
reopening separates geometric opportunity from injected information.
Performative shift and policy churn motivate testing beyond displacement
\citep{perdomo2020performative,schaul2022policychurn}, but neither distribution
shift nor policy change alone establishes quality conversion. Unlike information costs
used as control objectives \citep{todorov2006linearly}, KELB is a post hoc lower
certificate for a declared event-probability gain.

Sparse-reward exploration, demonstration learning, and recursive LLM feedback
optimize how agents improve
\citep{ecoffet2021first,ross2011reduction,schulman2017proximal,
madaan2023selfrefine,shinn2023reflexion}. A recent survey formalizes
self-improvement as durable updates to parameters or scaffolds and calls for
trajectory-, transfer-, and stability-aware evaluation
\citep{ren2026selfimprovements}. LOGO and CRSfD improve sparse-reward learning
via offline guidance or conservative reward shaping
\citep{rengarajan2022logo,guo2023mismatched}. We instead test decision-probe
change and certified quality conversion; the action permutation is a matched
negative control.

KELB is a classical event-level relative-entropy projection
\citep{kullback1951information,csiszar1975idivergence,cover2006elements} used here
as a finite-sample certificate retaining baseline difficulty. We use exact binomial
intervals for event probabilities and resample paired seeds for probe means
and performance contrasts \citep{clopper1934use,efron1993introduction};
Equation~\ref{eq:mi-decomposition} separates mutual information from change
relative to no intervention.

\section{Limitations and Conclusion}

Identifiability is probe-relative: RL action distributions and BO acquisition
decisions identify neither the full environment--learner kernel nor equivalent
parameters outside probe support. Our escape conditions are sufficient, not
necessary; KELB is a minimum baseline-relative cost, not consumed information;
and budget $m$ is only a proxy. RL uses one MiniGrid task and coarse grids. The
five-seed, three-replicate pilot (Supplementary Section~S1) finds aggregate
finite-horizon expansion with one draw-sensitive cell, but identifies neither a
local operator nor an asymptotic attractor. Its parameter perturbations are
coordinate-dependent, PPO continuation is stochastic, and the two quality
non-returns near the $.10$ threshold indicate binary evaluation sensitivity.
BO weighted sliced Wasserstein is also a decision proxy rather than a full
kernel identifier. Finer budgets, broader tasks and action permutations, and
confirmatory operator identification remain necessary. Persistence does not
show that the old attractor lost stability; attractor replacement remains an
untested extension requiring pre-specified return or local-operator tests around
both regimes.

The experiments provide partial empirical support for the observable
predictions of the information--reorganization--escape principle. In DoorKey,
weak targeted feedback changed the policy after continuation while held-out
success remained zero in the observed sample. At higher budgets, continuous
success distinguished interventions that the binary event scored identically.
In Hartmann--6, the ordering of the decision and quality budgets persisted as
the first certified decision budget moved across four grid points. These
findings distinguish the smallest tested budget with a detectable decision
change, the smallest with a certified quality gain, and the size of the
performance difference between budgets. Differences
in measurement precision can also contribute to certification order. The
experiments leave open how much of the separation reflects the underlying
learning dynamics.

In routing, exchanging only the pheromone map transferred part of the
later delivery benefit across branches on all three follow-up networks.
This identifies a state carrier under one routing rule. Whether changes
in RL policy or BO acquisition state play an analogous causal role requires
direct intervention on those systems.

Further progress should identify which persistent changes in effective
learning dynamics make useful behavior accessible, and whether the old regime
actually loses stability. The present results establish testable empirical
structure for that program without equating certification thresholds with
intrinsic transition points. One RL task, three controlled BO objectives,
finite grids, and approximate BO decision intervals limit generality.

\bibliographystyle{aaai2027}
\bibliography{reference}

\clearpage
\setcounter{section}{0}
\setcounter{secnumdepth}{1}
\renewcommand{\thesection}{S\arabic{section}}
\section*{Supplementary Material}
\paragraph{Relation to the first public version.}
Section 6.1 of \texttt{arXiv:2607.14185v1} reported an exploratory LLM
code-repair study with matched continuation and hidden-test outcomes.
Its task screening was exploratory and it did not measure a prespecified
kernel discrepancy. The present budget comparisons use the RL and BO
decision probes; the routing study tests state transfer separately.

\section*{Statistical Reanalysis and Historical Record}
This section reports the retrospective statistical reanalysis of the frozen
RL and BO experiments. The results in
this opening section supersede historical pointwise certificates for the main
paper's revised numerical claims. Sections S1--S5 reproduce the prior
supplementary record; their confidence intervals, thresholds, hypothesis
support statements, and robustness audits retain their original scope and are
not silently reinterpreted as family-adjusted findings.

\paragraph{Analysis status.}
The revised quality family has 124 comparisons, family error $.025$; the
primary decision family has 64 comparisons, family error $.025$. The latter
uses one million common seed-bootstrap draws within domain and approximate
percentile coverage. Five continuous RL means and ten paired contrasts form a
separate family at error $.05$. The original events remain primary frozen
estimands. BO same-query events and full margin curves are retrospective.
No models were retrained and no new BO trajectories were generated.

\paragraph{Reproducibility scope.}
All 1890 BO branch records and 22680 real-query records agree between the
archived CSV exports and the read-only SQLite database to numerical tolerance
$10^{-12}$. Every branch has 12 real-query rows. RL summaries cover all 150
seed--branch units and match archived event rows. Statistical summaries are
replayed from saved probe measurements; policies and GP probe measurements are
not independently regenerated. Across 40 listed frozen-lock entries, 39 match;
one RL operational-validation test file differs from its historical lock.
The listed core runtime and analysis entries match. Original source files were
not modified during the reanalysis. The discrepancy is disclosed without
attributing an unverified cause or claiming an entirely byte-identical tree.

\paragraph{Threshold interpretation.}
BO Styblinski--Tang generic first passed the original pointwise CP95 criterion
at $m=2$ (10/30 versus 0/30, KELB $.001609$), but the minimum leave-one-out
bound was zero. Its original table entry $4$ corresponded to the stable
threshold. The revised family-adjusted threshold is also $4$, for a different
reason. RL retains its frozen additional leave-one-out requirement.

\paragraph{Quality event reconstruction.}
The frozen BO event allows basin entry and the qualifying regret to occur at
different real queries. The stricter post-hoc event requires the same query.
All 69 discordances are original-positive/strict-negative on Two-basin,
including six baselines and 63 interventions. Counts are 385 versus 316 among
630 Two-basin branches, 56 versus 56 on Hartmann--6, and 431 versus 431 on
Styblinski--Tang. Strict-event inclusion does not imply a smaller
baseline-relative certificate, because both event probabilities change.

\begin{table*}[t]
\centering
\begin{tabular}{llccc}
\toprule
Task & Feedback & Adjusted $m_G$ & Original-event $m_Q$ & Same-query $m_Q$ \\
\midrule
Two-basin & Generic & 8 & --- & 16 \\
 & Mismatched & --- & --- & --- \\
 & Noisy-targeted & 1 & 16 & 2 \\
 & Targeted & 1 & 2 & 1 \\
Hartmann--6 & Generic & 1 & --- & --- \\
 & Mismatched & 1 & --- & --- \\
 & Noisy-targeted & 1 & --- & --- \\
 & Targeted & 1 & 16 & 16 \\
Styblinski--Tang & Generic & 1 & 4 & 4 \\
 & Mismatched & 1 & --- & --- \\
 & Noisy-targeted & 1 & 1 & 1 \\
 & Targeted & 1 & 1 & 1 \\
\bottomrule
\end{tabular}
\caption{Multiplicity-adjusted BO thresholds at the frozen $.01$ decision margin. Decision
coverage is approximate; all quality bounds use the 124-comparison family.
Same-query results are post-hoc. A dash is not a population threshold of infinity.}
\end{table*}

\paragraph{Complete margin sensitivity.}
For a fixed finite $m_Q$, $m_G(\varepsilon_G)<m_Q$ exactly when
$0\le\varepsilon_G<\max_{m<m_Q}L_m$, provided the maximum is positive.
The curves are deterministic transformations of the same fixed intervals,
not new independent tests at each margin. They do not assume that probe
response is monotone in budget. Under the original BO event, the positive
ranges end at $.207724127$ for Two-basin targeted, $.189650833$ for Two-basin
noisy-targeted, $.030445605$ for Hartmann--6 targeted, and $.156397809$ for
Styblinski--Tang generic. Bounds are exclusive. Their numerical precision is
not evidence of a natural transition constant; bootstrap endpoints have
sampling and Monte Carlo uncertainty.

Under the stricter event, positive ranges end at $.033148089$ for Two-basin
generic, $.177255817$ for Two-basin noisy-targeted, $.030445605$ for
Hartmann--6 targeted, and $.156397809$ for Styblinski--Tang generic.
Two-basin targeted has $m_Q=1$ and hence no strictly smaller tested decision
budget. No finite-order claim is made for missing quality certificates.

\begin{figure*}[t]
\centering
\includegraphics[width=\textwidth]{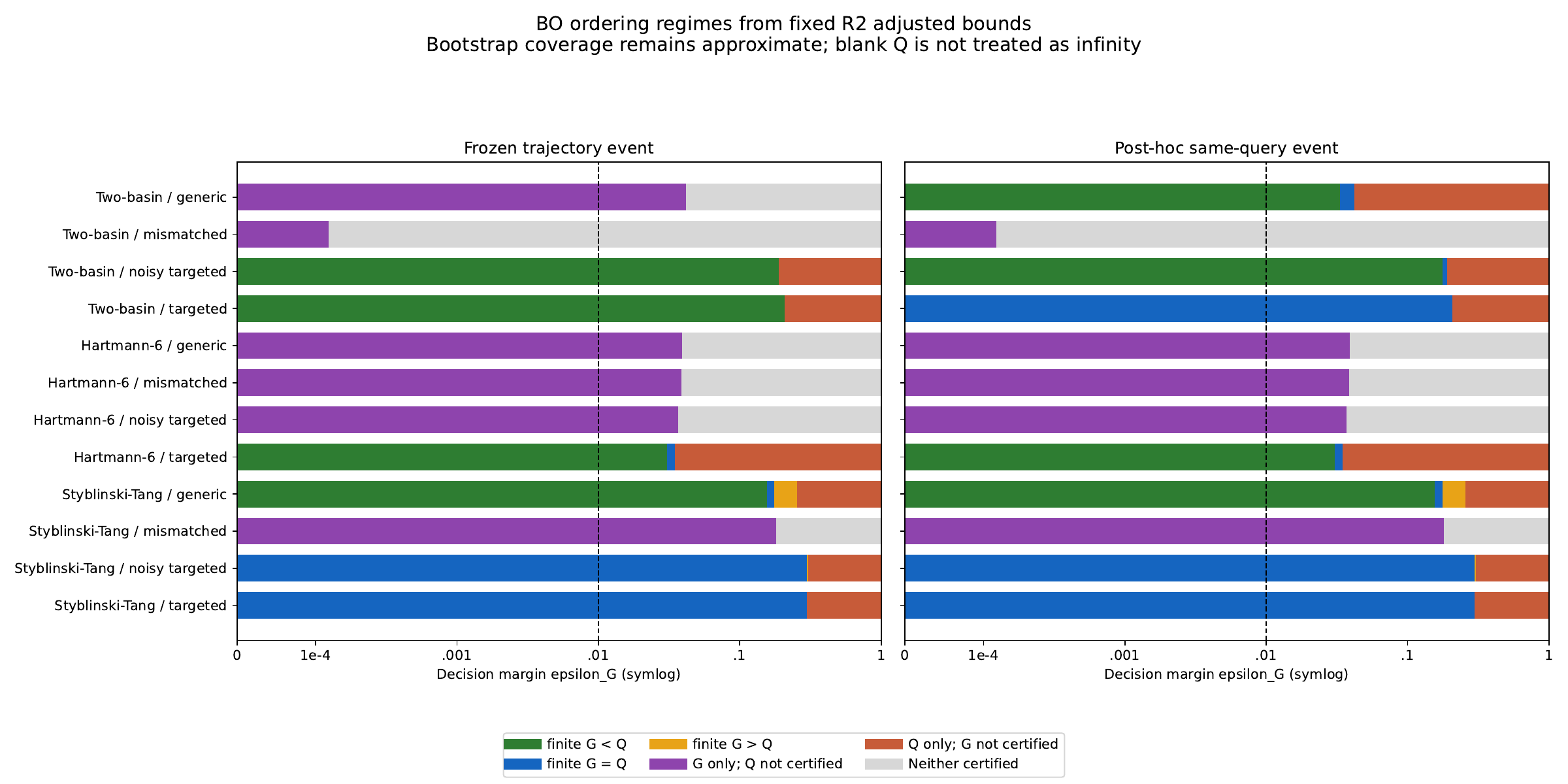}
\caption{BO decision-margin sensitivity after multiplicity adjustment on $[0,1]$, compared across the
frozen trajectory event and post-hoc same-query event. The horizontal scale
is symmetric logarithmic with a linear neighborhood of zero. The dashed line
marks $.01$. Colors distinguish finite threshold order from missing certificates.}
\end{figure*}

\paragraph{Finite-sample sensitivity and continuous RL quality.}
Hoeffding intervals on bounded decision values have half-width approximately
$.377291$ under the declared family. They provide no BO decision certificate
for any nonnegative margin. RL targeted still has $m_G<m_Q$ below $.223250142$,
compared with $.562858918$ under adjusted bootstrap intervals. Continuous
success at targeted $10000$ and $20000$ is $.224667$ and $.660333$, with
adjusted approximate intervals $[.205333,.242333]$ and $[.632000,.688333]$.
Their paired difference is $.435667$ $[.403333,.466667]$; targeted $20000$
minus mismatched is $.660333$ $[.632000,.688333]$. Zero-width bootstrap
intervals on all-zero samples do not demonstrate population certainty.

\paragraph{Historical material below.}
Sections S1--S5 preserve the earlier analyses for traceability,
including pilot evidence and pointwise sensitivity results. Their numerical
claims should be read under their original procedures. In particular, statements
about original CP95 values or $m_Q=8$ on Hartmann--6 are historical and do not
replace the revised main-paper thresholds.
\clearpage

\section{Exploratory Perturbation--Return Pilot}
\label{sec:perturbation-return}

\paragraph{Purpose and status.}
This exploratory pilot tests whether the finite-horizon response around the
saved RL checkpoints depends on a single parameter-noise realization. It is
separate from the prospectively frozen 30-seed confirmatory study in the main
paper and is not pooled with it. The independent statistical unit is the PPO
training seed; perturbation replicates and scales are repeated measurements,
not additional independent seeds.

\paragraph{Design.}
We use the five calibration seeds
$\{7,17,42,101,137\}$ and five checkpoint conditions: saturation, PPO-only
baseline, targeted $m=10000$, targeted $m=20000$, and mismatched $m=20000$.
For each seed--condition pair, we apply three independently sampled
parameter-noise perturbations at each relative scale
$\{.005,.01,.02\}$. This yields
$5\times5\times3\times3=225$ perturbation--continuation comparisons.
Noise is normalized separately for each nonzero floating-point policy tensor
and realizes the requested global relative parameter norm to within $10^{-8}$.
Each perturbed policy and its matched unperturbed reference receive the same
$20{,}000$-step sparse-reward PPO continuation. Policy total variation (TV) is
evaluated on the frozen 768-observation balanced task-phase support, and
quality is evaluated over 100 deterministic held-out episodes using the
pre-specified success threshold $.10$.

\paragraph{Diagnostic and analysis.}
For initial and final policy discrepancies
$\delta_{\mathrm{initial}}$ and $\delta_{\mathrm{final}}$, define the
finite-horizon perturbation-response ratio
\[
R=\frac{\delta_{\mathrm{final}}}{\delta_{\mathrm{initial}}}.
\]
Values $R>1$ denote expansion relative to the matched unperturbed
continuation, whereas $R<1$ denotes contraction. Because ratios are
right-skewed, the three replicates are averaged within each seed on the
$\log R$ scale. We then report the exponentiated mean and a 95\% bootstrap
interval over the five independent seeds. The draw count is descriptive and
is not treated as an independent-sample significance test.

\begin{table*}[t]
\centering
\small
\begin{tabular}{llrrrr}
\toprule
Condition & Scale & Geometric $R$ & 95\% seed-bootstrap CI &
$R>1$ draws & Quality return \\
\midrule
Saturation & .005 & 1311.21 & [869.54, 2197.17] & 15/15 & 15/15 \\
Saturation & .010 & 785.63 & [541.86, 1107.08] & 15/15 & 15/15 \\
Saturation & .020 & 383.53 & [267.52, 563.91] & 15/15 & 15/15 \\
PPO-only baseline & .005 & 737.93 & [666.50, 817.82] & 15/15 & 15/15 \\
PPO-only baseline & .010 & 368.54 & [230.36, 589.60] & 15/15 & 15/15 \\
PPO-only baseline & .020 & 172.02 & [121.00, 234.24] & 15/15 & 15/15 \\
Mismatched $m=20000$ & .005 & 11.64 & [8.85, 15.33] & 15/15 & 15/15 \\
Mismatched $m=20000$ & .010 & 6.57 & [4.55, 9.49] & 15/15 & 15/15 \\
Mismatched $m=20000$ & .020 & 3.62 & [2.97, 4.28] & 15/15 & 15/15 \\
Targeted $m=10000$ & .005 & 7.46 & [6.43, 8.48] & 15/15 & 14/15 \\
Targeted $m=10000$ & .010 & 3.65 & [2.91, 4.50] & 15/15 & 15/15 \\
Targeted $m=10000$ & .020 & 1.90 & [1.56, 2.21] & 15/15 & 14/15 \\
Targeted $m=20000$ & .005 & 5.65 & [5.07, 6.46] & 15/15 & 15/15 \\
Targeted $m=20000$ & .010 & 2.96 & [2.53, 3.43] & 15/15 & 15/15 \\
Targeted $m=20000$ & .020 & 1.22 & [1.11, 1.35] & 12/15 & 15/15 \\
\bottomrule
\end{tabular}
\caption{Exploratory finite-horizon perturbation--return results. Replicates
are averaged within each training seed on the log-ratio scale before
seed-level bootstrap resampling. Quality return denotes return to the
checkpoint's initial binary quality category, not recovery of the exact
success rate.}
\label{tab:perturbation-return}
\end{table*}

\paragraph{Results.}
Fourteen of the fifteen condition--scale cells expand in all 15 noise draws.
The exception is targeted $m=20000$ at scale $.02$: 12 of 15 draws expand and
three contract. Nevertheless, the within-seed mean log ratio is positive for
all five seeds, and its seed-aggregated geometric ratio is
$1.217$ with 95\% bootstrap interval $[1.114,1.354]$. Thus the aggregate
finite-horizon response is not attributable to a single noise draw, but the
largest-scale targeted cell remains perturbation-sensitive and should not be
described as universally expansive.

The perturbed branches return to the checkpoint's initial binary quality
category in 223 of 225 comparisons. Both non-returns occur for seed 42 under
targeted $m=10000$, near the pre-specified $.10$ success threshold. This
run-level rate is descriptive; it does not increase the independent sample
size beyond five seeds. Mean absolute final policy TV is $.071$ for saturation,
$.123$ for the PPO-only baseline, $.156$ for mismatched $m=20000$, $.067$ for
targeted $m=10000$, and $.036$ for targeted $m=20000$. These absolute values
are reported because the very small initial TV for saturation and baseline
inflates their ratios.

\paragraph{Interpretive boundary.}
This experiment is a finite-horizon perturbation--return diagnostic, not an
attractor-replacement test. Neural-network parameter perturbations are
coordinate-dependent; PPO continuation is stochastic and nonlinear; and the
pilot neither estimates a local transition operator nor establishes
asymptotic return. It therefore does not show that an old attractor lost
stability or that a governing operator was replaced. Its supported conclusion
is narrower: the aggregate perturbation-response pattern is not driven by one
parameter-noise draw, while quality-category behavior remains substantially
more stable than exact decision trajectories.

\section{Meaningful-Effect Margin Sensitivity}
\label{sec:epsilon-sensitivity}

\paragraph{Margin versus statistical power.}
The meaningful-effect margin $\varepsilon_G$ defines the resolution at which a
probe discrepancy is scientifically consequential. It is not a confidence
level or a test-size parameter. Statistical precision is represented by the
width of the seed-bootstrap interval. For any fixed interval
$[L_m,U_m]$, a budget is classified as reorganized when
$L_m>\varepsilon_G$, equivalent at that resolution when
$U_m<\varepsilon_G$, and inconclusive otherwise. Thus a larger margin demands
a larger effect for reorganization but also makes equivalence easier to
certify. All classifications below reuse the same frozen intervals; no outcome
is re-estimated for a preferred threshold.

\paragraph{RL margin selection.}
The five-seed calibration protocol considered
$\varepsilon_G\in\{.02,.05,.10\}$ before the 30 confirmatory seeds were run.
We selected the most conservative candidate, $.10$, and froze it for the
primary balanced-support policy-TV estimand. Table~\ref{tab:epsilon-sensitivity}
shows that all three candidate margins give the same targeted threshold
$m_G=1000$; mismatched $m=20000$ is also classified as reorganized under all
three margins. The quality threshold is defined by KELB and therefore does not
change with $\varepsilon_G$: targeted $m_Q=10000$ throughout.

\paragraph{BO frozen sensitivity.}
The BO protocol froze $.01$ as primary and $.005,.02$ as sensitivity margins.
Controlled two-basin and Styblinski--Tang thresholds are unchanged throughout
this range. On Hartmann--6, the larger $.02$ margin moves several decision
thresholds to larger tested budgets. In particular, targeted $m_G$ moves from
$1$ to $4$, while the historical pointwise quality threshold is
$m_Q=8$. The family-adjusted analysis in the main text gives $m_Q=16$;
both analyses place the quality threshold above the decision threshold.

\begin{table*}[t]
\centering
\small
\setlength{\tabcolsep}{4pt}
\begin{tabular}{llccc}
\toprule
Study/task & Condition & Margin 1 & Margin 2 & Margin 3 \\
\midrule
RL DoorKey & Targeted & $1000$ $(.02)$ & $1000$ $(.05)$ & $1000$ $(.10)$ \\
RL DoorKey & Mismatched & $20000^{\dagger}$ $(.02)$ & $20000^{\dagger}$ $(.05)$ & $20000^{\dagger}$ $(.10)$ \\
\midrule
Two-basin & Generic & $8$ $(.005)$ & $8$ $(.01)$ & $8$ $(.02)$ \\
 & Mismatched & --- & --- & --- \\
 & Noisy-targeted & $1$ & $1$ & $1$ \\
 & Targeted & $1$ & $1$ & $1$ \\
Hartmann--6 & Generic & $1$ & $1$ & $2$ \\
 & Mismatched & $1$ & $1$ & $2$ \\
 & Noisy-targeted & $1$ & $1$ & $4$ \\
 & Targeted & $1$ & $1$ & $4$ \\
Styblinski--Tang & Generic & $1$ & $1$ & $1$ \\
 & Mismatched & $1$ & $1$ & $1$ \\
 & Noisy-targeted & $1$ & $1$ & $1$ \\
 & Targeted & $1$ & $1$ & $1$ \\
\bottomrule
\end{tabular}
\caption{$m_G$ sensitivity using the frozen seed-bootstrap intervals. Parentheses
give $\varepsilon_G$ where needed. RL margins are $.02,.05,.10$ from left to
right, with $.10$ primary; BO margins are $.005,.01,.02$, with $.01$ primary.
A dash denotes no reorganized budget on the tested grid.
$^{\dagger}$Only $m=20000$ was tested for mismatched RL, so this is a certified
tested budget, not a minimum over lower mismatched budgets.}
\label{tab:epsilon-sensitivity}
\end{table*}

\paragraph{Descriptive RL breakpoints.}
For transparency beyond the pre-specified candidate margins, the fixed
balanced-support lower endpoints imply exact post-hoc resolution breakpoints.
For targeted feedback, $m_G=1000$ for
$\varepsilon_G<.5795$, $m_G=10000$ for
$.5795\le\varepsilon_G<.7001$, $m_G=20000$ for
$.7001\le\varepsilon_G<.7783$, and no tested targeted budget is certified for
$\varepsilon_G\ge.7783$. For mismatched $m=20000$, the 95\% interval is
$[.8146,.8362]$: it is reorganized below $.8146$, inconclusive within that
interval, and equivalent above $.8362$. These breakpoints are descriptive
functions of revealed intervals, not alternative confirmatory specifications.

\section{RL Confirmatory Protocol and Reproducibility}
\label{sec:rl-protocol}

\paragraph{Independent units and branch construction.}
The environment is \texttt{MiniGrid-DoorKey-8x8-v0}. The 30 confirmatory
training seeds are the odd integers from 1001 through 1059; calibration seeds
$\{7,17,42,101,137\}$ are excluded. Each sparse-reward PPO run is evaluated
every $10{,}000$ steps and declared saturated when the success-rate range over
the last five evaluations is at most $.05$, after at least $50{,}000$ and at
most $150{,}000$ steps. The saturated checkpoint is branched into PPO-only
continuation, targeted BC+PPO at $m\in\{1000,10000,20000\}$, and mismatched
BC+PPO at $m=20000$. The mismatched branch applies the fixed action permutation
$[1,2,0,4,5,3,6]$ globally. Every branch then receives $100{,}000$ PPO steps
under the original sparse reward with a reinitialized optimizer.

\begin{table}[t]
\centering
\small
\setlength{\tabcolsep}{4pt}
\begin{tabular}{ll}
\toprule
Component & Frozen setting \\
\midrule
PPO learning rate & $2.5\times10^{-4}$ \\
Rollout / batch size & $1024$ / $256$ \\
PPO epochs & $4$ \\
$\gamma$ / GAE $\lambda$ & $.99$ / $.95$ \\
Entropy / value coefficient & $.01$ / $.5$ \\
BC epochs / batch size & $200$ / $256$ \\
BC learning rate & $.005$ \\
Post-branch continuation & $100{,}000$ PPO steps \\
Evaluation & $100$ common held-out layouts \\
Decision support & $3\times256$ observations \\
Bootstrap & $10{,}000$ seed resamples \\
\bottomrule
\end{tabular}
\caption{Frozen RL training and evaluation settings.}
\label{tab:rl-hyperparameters}
\end{table}

All branches of a training seed use branch-randomness seed
$2{,}000{,}000+s$. Evaluation layouts are shared across branches and use
evaluation-layout seed $3{,}000{,}000+s$. The balanced decision support contains 256
observations from each pre-key, key-to-door, and door-to-goal phase. The
primary discrepancy is mean categorical TV against the matched PPO-only
continuation. The seed-level quality event is success rate at least $.10$ over
100 held-out layouts; $.05$ and $.20$ are frozen quality sensitivities.

\paragraph{Certification and confirmatory outcomes.}
The training seed is resampled as the independent unit. KELB uses a
Bonferroni joint 95\% Clopper--Pearson rectangle for intervention and baseline
event rates, and every positive certificate is refit after leaving out each
seed. All 30 matched baselines have zero quality events. Targeted $m=1000$ has
policy-TV $.601$ with 95\% interval $[.580,.622]$ but 0/30 quality events.
Targeted $m=10000$ and $m=20000$ each have 30/30 events versus 0/30 baselines,
with certified KELB lower bound $1.347$ and complete leave-one-out stability.
Mismatched $m=20000$ has policy-TV $.825$ $[.815,.836]$ and 0/30 events.
Consequently $m_G=1000<m_Q=10000$, and at $m=20000$ the paired
targeted-minus-mismatched event difference is $1.00$ with interval
$[1.00,1.00]$. The same alignment direction holds at quality thresholds
$.05,.10,.20$ and the reorganization classifications hold on all three frozen
alternative supports.

\section{BO Confirmatory Protocol and Probe Audits}
\label{sec:bo-protocol}

\paragraph{Design.}
The tasks are controlled two-basin, two-dimensional Styblinski--Tang, and
Hartmann--6. The 30 confirmatory seeds are
\[
\begin{aligned}
\{&1009,1013,1019,1021,1031,1033,1039,\\
  &1049,1051,1061,1063,1069,1087,1091,\\
  &1093,1097,1103,1109,1117,1123,1129,\\
  &1151,1153,1163,1171,1181,1187,1193,\\
  &1201,1213\}.
\end{aligned}
\]
Smoke-test seeds are excluded. Conditions
are baseline, generic, mismatched, noisy-targeted, and targeted, with
$m\in\{1,2,4,8,16\}$ and baseline $m=0$. Phase 1 begins with eight local
points, runs for 12--35 rounds, and uses a five-round saturation window. After
the full domain is reopened, every branch receives 12 real expected-improvement
queries. Injected observations update the Gaussian-process posterior but are
not counted as real queries or improvements.

\begin{table}[t]
\centering
\small
\setlength{\tabcolsep}{4pt}
\begin{tabular}{ll}
\toprule
Component & Frozen setting \\
\midrule
Acquisition & Expected improvement \\
GP observation noise & $10^{-6}$ \\
GP optimizer restarts & $0$ \\
Acquisition temperature $\tau$ & $.05$ \\
Phase-1 reference points & $512$ \\
Full-domain reference points & $2500$ \\
Posterior samples & $512$ \\
Decision projections & $32$ fixed \\
Post-intervention queries & $12$ real queries \\
Noisy-targeted std. / bias & $.35$ / $.25$ \\
Primary $\varepsilon_G$ & $.01$ \\
Sensitivity margins & $.005,.02$ \\
\bottomrule
\end{tabular}
\caption{Frozen BO settings.}
\label{tab:bo-hyperparameters}
\end{table}

The decision probe is normalized weighted sliced Wasserstein distance between
pre- and post-feedback expected-improvement weights on the fixed full-domain
support. A quality event requires an active real query in the declared global
basin and best post-intervention regret no worse than the nominal local
attractor. The primary CP95 KELB certificate and every reported positive
threshold use the same leave-one-seed-out requirement as RL.

\paragraph{Threshold outcomes.}
Targeted feedback has $m_G=m_Q=1$ for controlled two-basin and
Styblinski--Tang. On Hartmann--6, targeted feedback has $m_G=1$ and $m_Q=8$
at the primary margin; at $\varepsilon_G=.02$, $m_G=4<m_Q=8$. At $m=16$,
Hartmann targeted feedback produces 25/30 quality events versus 0/30 baseline
events, with CP95 KELB lower bound $.645$ and complete leave-one-out stability.
For controlled two-basin and Styblinski--Tang, targeted $m=1$ produces 26/30
versus 6/30 and 30/30 versus 0/30 events, respectively, with certified lower
bounds $.133$ and $1.347$. Paired aligned-minus-nonaligned quality-event
differences are $.590$ $[.517,.660]$, $.180$ $[.147,.213]$, and $.563$
$[.517,.607]$ for the three tasks.

\paragraph{Secondary Phase-1 spectral diagnostic.}
For each task, we form a frozen observable from the normalized acquisition
barycenter and argmax, normalized entropy, maximum acquisition probability,
and incumbent regret. After feature standardization and within-seed centering,
a ridge VAR(1) with penalty $10^{-3}$ is fitted to the last eight Phase-1
rounds. Thirty seeds contribute 210 one-step transitions, and 95\% intervals
resample seeds 2000 times. Table~\ref{tab:bo-spectral} reports spectral radii
well below one for this observable. This is local spectral support over the
observed Phase-1 window; coordinate choice, finite horizon, regularization,
and residual model error prevent interpreting it as global convergence or
attractor identification.

\begin{table}[t]
\centering
\small
\setlength{\tabcolsep}{3pt}
\begin{tabular}{lccc}
\toprule
Task & Radius & 95\% CI & Innovation RMS \\
\midrule
Two-basin & $.147$ & $[.147,.169]$ & $.335$ \\
Hartmann--6 & $.368$ & $[.342,.521]$ & $.792$ \\
Styblinski--Tang & $.132$ & $[.077,.176]$ & $.427$ \\
\bottomrule
\end{tabular}
\caption{Secondary ridge-VAR(1) diagnostic on the frozen Phase-1 acquisition
observable. ``Radius below one'' is a finite-window local classification, not
an attractor-replacement test.}
\label{tab:bo-spectral}
\end{table}

\paragraph{Decision-probe audit.}
The frozen 32-projection statistic was reconstructed with maximum absolute
error $9.99\times10^{-17}$. Across full-domain supports containing
$128,256,512,1000,$ or $2500$ candidates, 62/63 cells retain their primary
classification; the sole exception is controlled generic $m=4$, adjacent to
the $.01$ boundary. All preregistered focal contrasts are stable at every
candidate count. All 63 cells retain their classification under
$8,16,32,64,$ and $128$ projections. Sliced-Wasserstein rankings correlate
with Jensen--Shannon, TV, and random-feature MMD at Spearman
$\rho=.959,.953,$ and $.998$, respectively.

Repeated seed subsampling further separates margin choice from statistical
precision. Recovery of the primary targeted $m=1$ classification is 1.00 at
$n=5,10,15,20,25,30$ for two-basin and Styblinski--Tang. Hartmann--6 is the
near-margin case: recovery is $.380,.613,.856,.978,1.00,1.00$ at those sample
sizes. Thus its sensitivity is visible both as a larger-$\varepsilon_G$
threshold shift and as lower recovery with small seed counts.

Support changes alter the estimand rather than merely Monte Carlo precision.
Full-domain, balanced local/global, and global-boundary supports classify
52/63, 43/63, and 32/63 cells as reorganized; the incumbent-local support
classifies all 63 as equivalent. In particular, targeted $m=1$ is reorganized
on the frozen full-domain support but equivalent on the incumbent-local support
for all three tasks. These audits support reproducibility of the declared
probe, not a support-independent full-kernel claim.

\section{Machine-Readable Result Inventory}
\label{sec:data-inventory}

The reproducibility package organizes seed-level files rather than treating
episodes, candidate points, or perturbation draws as independent observations.
For RL, \texttt{confirmatory\_policy\_probe\_rows.csv} contains decision
discrepancies and \texttt{confirmatory\_quality\_event\_rows.csv} contains
paired quality events. For BO,
\texttt{bo\_probe\_kernel\_discrepancy\_rows.csv} contains probe rows and
\texttt{bo\_kelb\_joint\_seed\_rows.csv} contains matched event rows. Summary
tables, frozen protocol documents, configuration files, and analysis metadata
are stored alongside them. Calibration and smoke-test outcomes are stored
separately and are never pooled with confirmatory estimates.

For the historical RL and BO analyses in Sections S1--S5,
the measured outcomes are decision-probe reorganization and finite-horizon quality
conversion. The perturbation--return pilot adds an exploratory stability
diagnostic but does not establish attractor replacement, asymptotic
persistence, or full-kernel reorganization.

\section{Self-Play Pilots: Testing for an Eligible Plateau}
An externally exploitable strategy can still be improving under self-play.
We used exact exploitability (half the sum of the two best-response gains)
to test whether intervention was warranted. In Kuhn poker we enumerated
64 pure plans per seat over all six deals. The evaluator recovered
exploitability $11/24$ for uniform behavioral play, zero for an analytic
equilibrium, and equilibrium first-seat value $-1/18$ chip per hand.

In full-tree counterfactual regret minimization, mean average-policy
exploitability over eight initializations fell from .01616 at 256 updates
to .00363 at 4096. The last three checkpoints (1024, 2048, 4096) had to
remain at or above .02 with range at most .005. No seed qualified.

A second pilot held the exponential-weights update rule fixed in
normal-form Kuhn and six nontransitive zero-sum matrix games (12 actions
in three cyclic families; instance seeds 100--105). Full feedback provided
conditional expected payoffs of every plan against the sampled opposing
plan. Sampled feedback used the realized payoff of the played plan,
weighted by inverse selection probability. Both used learning rate .004
and uniform exploration .1. The arms matched rounds, with greater
counterfactual payoff access in the full-feedback arm.
At 4096, 8192 and 16384 rounds the average-policy gap had to remain
at least .04 with range at most .01, measured in chips per hand for
Kuhn and game payoff units for the matrices. In sampled Kuhn, the
seed means were .2269, .1980 and .1562; with full feedback they were
.2075, .1715 and .1122. At the final round the six matrix instance
means ranged from .0455 to .0551 with sampled feedback and .0406 to
.0510 with full feedback. Zero of 16 seeds qualified in each game
and feedback condition. The prespecified intervention stage required
at least 12 qualifying sampled seeds per game, and four qualifying
instances for the matrix family. It was never triggered. These finite
windows illustrate the distinction between remaining exploitable and
ceasing to improve; the pilots supply no estimate of an external
opponent's intervention effect.

\section{Routing Network Specifications}
Capacities are listed in terminal-route order, left branch first.
For A the vectors before and after the switch were $(24,4,2,2,2)$
and $(6,4,4,4,20)$. For B and C they were $(24,4,3,2,2,2)$
and $(6,4,3,4,5,20)$. In C the last route on each side shared
a downstream capacity of 12; terminal capacities were applied first
and independent random priorities allocated shared capacity.
The new high-capacity route was the final route in each vector.
Simulations started with zero load and pheromone. The random arm
sampled four routes independently and uniformly with replacement.
Seed ranges were 10000--10127, 11000--11127 and 12000--12127
for A, B and C. Each seed supplied both map contrasts. The local
follow-up protocol specified positive simultaneous lower bounds for
both contrasts in every network as the criterion for replication
across the three networks.

\end{document}